\documentclass{article} %
\usepackage{iclr2027_conference,times}

\usepackage{amsmath,amsfonts,bm}

\def\eqref#1{equation~\ref{#1}}

\def\1{\bm{1}}

\DeclareMathAlphabet{\mathsfit}{\encodingdefault}{\sfdefault}{m}{sl}
\SetMathAlphabet{\mathsfit}{bold}{\encodingdefault}{\sfdefault}{bx}{n}

\usepackage{latexsym}
\usepackage{amsthm}

\usepackage{graphicx}
\usepackage{hyperref}
\usepackage{bm}
\usepackage{url}
\usepackage{wrapfig}
\usepackage{graphicx} 
\usepackage{booktabs}
\usepackage{tabularx}
\usepackage{xcolor}
\usepackage{amssymb}
\usepackage{amsmath}
\usepackage{array}
\usepackage{cleveref}
\usepackage{tikz}
\usepackage{caption}

\crefformat{equation}{equation~#2#1#3}
\Crefformat{equation}{Equation~#2#1#3}
\crefname{appendix}{appendix}{appendices}
\Crefname{appendix}{Appendix}{Appendices}

\usepackage{thmtools} 
\usepackage{thm-restate}

\declaretheorem[name=Lemma,shaded={bgcolor={blue!10!white}}]{lmm}

\newcommand{\restate}[1]{\csname #1\endcsname*}

\usepackage{todonotes}

\usepackage{xspace}
\newcommand{\com}{\texttt{COM}\xspace}
\newcommand{\comemb}{\texttt{COM-EMB}\xspace}

\title{There is No Theoretical Curse of Multilinguality For Embedding Space Structure}

\author{
Niyati Bafna$^{1}$ \quad
Neha Verma$^{1}$ \quad
Vilém Zouhar$^{2}$ \quad
Philipp Koehn$^{1}$ \quad
David Yarowsky$^{1}$
\\[0.6em]
\hspace{0.3cm}$^{1}$Johns Hopkins University, Center for Language and Speech Processing\quad
$^{2}$ETH Zürich\\[0.6em]
\hspace{5.3cm}\texttt{nbafna1@jhu.edu}
}

\iclrfinalcopy %
\begin{document}
\maketitle
\lhead{}
\begin{abstract}
A central goal of multilingual NLP is to achieve high monolingual performance per language and cross-lingual alignment for large-scale language coverage with a multilingual model.
The \emph{curse of multilinguality} describes the phenomenon of degradation in multilingual model performance as we increase language coverage, posing a threat to the above goal. 
This paper asks whether multilingual embedding spaces are inherently incapable of achieving perfect multilinguality without a prohibitive increase in required capacity.
We first formalize the goal of ``perfect multilinguality'', embodied in two \emph{multilinguality conditions}.
We then prove that the minimum dimensionality required for perfect multilinguality grows only logarithmically, as $\Theta(\log L)$, in the number of languages $L$.
That is, we show that there is no theoretical curse of multilinguality for embedding space structure.
This suggests that the empirical curse of multilinguality is a result of real world data and training conditions.
We back this understanding with a small-scale empirical study.
Our paper provides the first theoretical and intrinsic perspective on the curse of multilinguality.\footnote{The code for our experiments is available at \href{https://github.com/niyatibafna/curse-of-multilinguality}{github.com/niyatibafna/curse-of-multilinguality}.}
\end{abstract}

\section{Introduction}
The ideal multilingual model exhibits high monolingual performance on individual languages as well as perfect  cross-lingual structure, or semantic sharing across languages, for arbitrary language coverage.
The well-known \emph{curse of multilinguality} (\com, \citealp{conneau-etal-2020-unsupervised}), also sometimes termed \emph{negative interference} \citep{shaham-etal-2023-causes}, describes the phenomenon of observed degradation of multilingual performance as we train a fixed-capacity model on increasingly more languages. 
Folk wisdom generally understands fixed capacity to be an inherent bottleneck \citep{gurgurov-etal-2024-multilingual}, resulting in ``less space'' for each language, indicating that perfect multilinguality as described above may not be achievable for high language coverage.
Several previous works have demonstrated the \com from the perspective of extrinsic task performance \citep{conneau-etal-2020-unsupervised,aharoni-etal-2019-massivel}, and while other works have explored strategies to mitigate this issue \citep{wang-etal-2020-balancing,pfeiffer-etal-2022-lifting}, it remains a serious challenge.
This gives rise to a fundamental question: is it inherently impossible to accommodate increasing language coverage in a limited-capacity model while maintaining desired monolingual and cross-lingual performance on all languages?
We provide a theoretical and intrinsic perspective on the \com.
Specifically, we study the curse of multilinguality for embedding space structure (\comemb), and show that multilingual representations can exhibit intrinsic quality while accommodating more languages without a prohibitive increase in capacity. 

First, we formulate two intrinsic requirements for a \emph{perfect multilingual} embedding space, drawing from consensus in the field \citep{chang-etal-2024-when,hammerl-etal-2024-understanding}, and term them \emph{multilinguality conditions}: (a) the space exhibits high-quality shared monolingual semantics for each language, and (b) the space exhibits cross-lingual alignment, meaning that translation equivalents across languages are ``close'' to each other in the space. 
This is important for applications such as cross-lingual information retrieval \citep{sasaki2018cross,roy-etal-2020-lareqa}. 
In the context of embedding spaces, ``capacity'' can be understood as dimensionality.
Then, our theoretical result shows that the minimum dimensionality of a perfect multilingual space encoding $L$ languages grows as $\Theta(\log L)$.
Specifically, we assume a base $D$-dimensional concept space and show that the minimum sufficient extra dimensionality to accommodate $L$ languages while maintaining the multilinguality conditions grows as $O(\log L)$, and that the dimensionality of any such space grows at least as $\Omega(\log L)$.
This result shows that multilinguality is not entirely free: it requires an additional capacity cost depending on $L$. 
However, given that the number of languages in the world is finite and usually estimated at around $\sim$7000 \citep{eberhard2021ethnologue}, the log dependence of the cost indicates only a small multilinguality tax as we increase language coverage, provably as small as $30$ extra dimensions for $7000$ languages. 
Thus, we show that there is no theoretical curse of multilinguality for embedding space structure.

The above finding suggests that any empirical curse of multilinguality for embedding space structure is a consequence of real-world conditions such as data and compute constraints, among other things.
We characterize the empirical \comemb via a controlled study of the empirical degradation of the multilinguality conditions with increasing language coverage in Transformer-based models, in eight different configurations varying compute constraints, sampling conditions, and the set of languages being evaluated. 
We show that while we do observe ``cursed'' behaviour in line with general wisdom, this behaviour depends on these configurations and is mitigated in favourable conditions.

Our paper contributes the first theoretical and intrinsic grounding to the curse of multilinguality. 
In showing that perfect intrinsic multilinguality is theoretically achievable at an arbitrary language scale, we set the bar for empirical multilingual spaces, and promote future work in exploring and closing the gap.

\section{Background}

\paragraph{Dimensionality bounds on embedding spaces} 
Recent work studies the embedding dimension required to express top-$k$ subsets of $N$ documents by single-vector retrieval. \citet{weller2026on} show that, given a fixed margin of separation, the necessary embedding dimension has a logarithmic dependence on the total number of subsets $\binom{N}{k}$, using classical sphere packing arguments \citep{vershynin2019high,conway1999recent}.
\citet{wang2026mathbb} show that $d=2k+1$ is a sufficient dimension for expressing all top-$k$ subsets, independent of $N$, and \citet{okajima2026limits} subsequently show that finite-precision quantization introduces a dependence on $N$ for the necessary dimension, which must grow logarithmically with the corpus size for a fixed precision.
Our work connects this line of study with fundamental questions in multilingual NLP.

\paragraph{Curse of multilinguality}
The ``curse of multilinguality'' or phenomenon of ``negative interference'' is understood as the observed degradation in multilingual model performance as the number of training languages is increased for a fixed capacity model, commonly described in terms of capacity dilution over languages \citep{aharoni-etal-2019-massivel,pfeiffer-etal-2022-lifting}.
While multilinguality provides benefits up to a point especially for low-resource languages via cross-lingual transfer from other language training data, high-resource language performance usually suffers; further, after a point, all language performance suffers from including more training languages, constituting a ``transfer-dilution'' tradeoff \citep{arivazhagan-etal-2019-massively,chang-etal-2024-when,gurgurov-etal-2024-multilingual}.
\citet{shaham-etal-2023-causes} show that the proportion of the target language with respect to total training data is an important factor, and that increasing model size helps, raising the question of the minimum model size required to accommodate a given language coverage.
\citet{longpre-etal-2026-atlas} develop empirical scaling laws for model and data size while increasing language coverage.
There are several works that attempt to mitigate the phenomenon of degradation, including strategies for better data sampling \citep{wang-etal-2020-balancing,kreutzer-etal-2021-banditsa,foroutan-etal-2025-revisiting}, architectural innovations such as mixture-of-expert models \citep{pfeiffer-etal-2022-lifting}, and optimization techniques such as Gradient Vaccine \citep{wang-etal-2020-gradienta}.
The above works explore the curse of multilinguality from a purely empirical lens.
To the best of our knowledge, we are the first to provide a theoretical lens on the curse of multilinguality.

\paragraph{Intrinsic characterization of embedding spaces}

Previous works in the curse of multilinguality study embedding space quality only indirectly by measuring extrinsic performance on tasks such as XNLI \citep{conneau2018xnli} and machine translation \citep{aharoni-etal-2019-massivel}.
However, there is broad general interest in investigating the structure of multilingual spaces, including assessing cross-lingual overlap and alignment \citep[inter alia]{xu-etal-2023-structural,wen2023hyperpolyglot,shah-etal-2023-geometry,mousi-etal-2024-exploring,hammerl-etal-2024-understanding}, and the encoding of language versus semantic information \citep{chang-etal-2022-geometrya,xie2022discovering}.
Tasks such as monolingual and cross-lingual retrieval and bitext mining depend directly on representation structure \citep{roy-etal-2020-lareqa,enevoldsen2025mmteb}, and inspire the formulation of our multilinguality conditions (\Cref{sec:defining_conditions}).
However, to the best of our knowledge, no previous work seeks to understand the curse of multilinguality from the intrinsic perspective of its impact on embedding space structure.

\section{Multilinguality Conditions for Perfect Multilingual Spaces}
\label{sec:defining_conditions}

In this section, we formulate conditions necessary for perfect multilinguality, given a reference high-quality monolingual concept space.

\vspace{-3mm}

\paragraph{Preliminaries}
We begin by defining key structures relevant to our argument. 
Let $\mathcal{C}$ be a global set of concepts, and let $\mathcal{L}$ be a set of languages.
Let $Z$ be the reference concept space, where $c \in \mathcal{C}$ is embedded as $z_c$ with high-quality semantic relationships.
For the multilingual embedding space $X^L$ encoding $L=|\mathcal{L}|$ languages, let $x_{c,\ell}$ represent the embedding for concept $c \in \mathcal{C}$ in language $\ell \in \mathcal{L}$.
Let $X^L_\ell \subset X^L$ denote the monolingual subspace of language $\ell \in \mathcal{L}$.
We use $\cos(\cdot, \cdot)$ to refer to the cosine similarity operator.
Finally, we use $\gamma$ to refer to the resolution scale of $Z$ by which any two distinct points are minimally separated: i.e. we have $\gamma<1$ such that $\forall c, c' \in \mathcal{C}, c\neq c', \cos(z_c, z_{c'}) < \gamma$.

\paragraph{Monolingual structure} 
A perfect multilingual space should encode high-quality shared monolingual relationships for each encoded language.
We use the high-quality concept space $Z$ as reference.
Then, perfect monolingual structure requires that for each language $\ell$ represented in space $X^L$, its monolingual structure $X^L_\ell$ exactly mimics $Z$.\footnote{In practice, we do not know what perfect concept relationships are. However, we can consider the monolingual space of a high-resource language (e.g. English) as a reference space, making the monolingual structure condition equivalent to requiring parity for monolingual quality and shared semantics across languages.}

Monolingual structure can be understood in terms of relative similarity judgments between pairs of concepts.
Thus, perfect monolingual structure can be understood as the monolingual subspace of each language exhibiting identical similarity rankings as the reference concept space, resulting in high-quality and shared semantics across languages.
\begin{restatable}[Monolingual Structure]{dff}{dff:monolingual-structure}
\label{dff:monolingual-structure}
A multilingual space $X^L$ satisfies the monolingual structure condition if each encoded monolingual subspace in it is identical in terms of similarity rankings ($\sim_R$) to a high-quality concept space $Z$, i.e. $\forall \ell \in \mathcal{L}, X^L_\ell \sim_R Z$. 
This condition requires that $\forall \ell \in \mathcal{L}$ and $c, c', d, d' \in \mathcal{C}$,
$\cos(z_c,z_{c'}) \le \cos(z_d,z_{d'}) \iff \cos(x_{c,\ell},x_{c',\ell}) \le \cos(x_{d,\ell},x_{d',\ell})$.
\end{restatable}

\newcommand{\word}[1]{\emph{#1}}

\begin{center}
    \begin{minipage}{0.35\linewidth}
    \vspace{-5mm}
    \includegraphics[width=\linewidth]{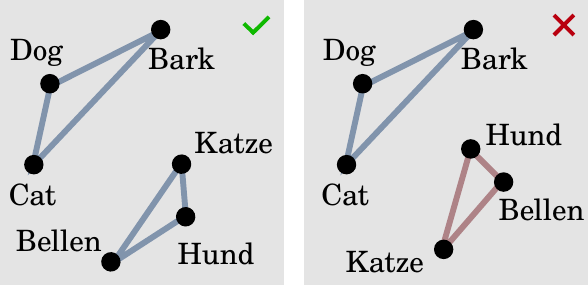}
    \end{minipage}
    \hfill
    \begin{minipage}{0.63\linewidth}
    \captionof{figure}{Illustration of \Cref{dff:monolingual-structure}. \word{Dog=Hund}, \word{Cat=Katze}, \word{Bark=Bellen}. The concept space is depicted in English. The German space on the left exhibits identical ordering of neighbours with the concept space while the German space on the right does not, because the nearest neighbour of \word{Katze(=Cat)} is \word{Bellen(=Bark)} and not \word{Hund(=Dog)}.}
    \label{fig:definition_1}
    \end{minipage}
\end{center}

\paragraph{Cross-lingual alignment}

A perfect multilingual space displays cross-lingual alignment, understood as the property that language variants of a given concept are close together, and farther from other concepts.
Past literature distinguishes a ``strong'' and ``weak'' view of this idea \citep{roy-etal-2020-lareqa,hammerl-etal-2024-understanding}. 
In the ``strong'' view, we require a concept in language $\ell$ to be closer to its translation equivalent in language $\ell'$ than any different concept in \textit{any} language, including other concepts in $\ell$, whereas in the weak view, it is only required to be closer than any other $\ell'$-language points.
We use the strong view, which also trivially implies the weak one.

\begin{restatable}[Cross-lingual alignment]{dff}{dff:cross-lingual-alignment}
\label{dff:cross-lingual-alignment}
A multilingual space $X^L$ satisfies the cross-lingual alignment condition if a concept encoding in a given language is closer to the encoding of the same concept expressed in any other language than to the encoding of any other concept in any language.
That is, this condition requires that for all concepts $c, c' \in \mathcal{C}$ where $c\neq c'$ and languages $\ell, m, n \in \mathcal{L}$, $\cos(x_{c,\ell},x_{c,m}) > \cos(x_{c,\ell},x_{c',n})$.
\end{restatable}
\begin{center}
    \begin{minipage}{0.35\linewidth}
    \vspace{-1mm}
    \includegraphics[width=\linewidth]{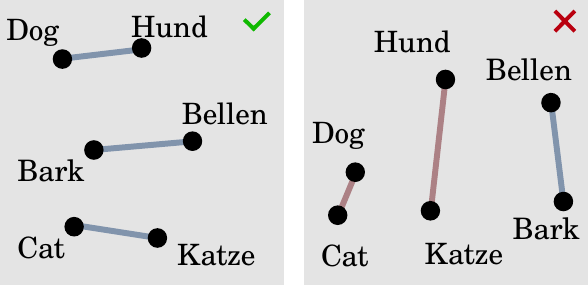}
    \end{minipage}
    \hfill
    \begin{minipage}{0.63\linewidth}
    \captionof{figure}{Illustration of \Cref{dff:cross-lingual-alignment}, with two languages. In the left space, \word{Hund(=Dog)} is closer to \word{Dog} than any different-concept word, but in the right space, the different-concept word \word{Cat} is closer to \word{Dog} than \word{Hund=(Dog)}.}
    \label{fig:definition_2}
    \end{minipage}
\end{center}

\paragraph{Non-degeneracy.}
A multilingual space can place all concept translation equivalents as co-incident at the same point in the reference high-quality monolingual space, and thus trivially exhibit monolingual structure as well as cross-lingual alignment.
However, this violates common practical assumptions of embedding spaces where an embedding space must be capable of distinguishing between any two embedded points with some fixed scale of resolution \citep{weller2026on}. 
Thus, we require that our multilingual space have a resolution scale, or a maximum cosine similarity $\gamma_{X}<1$ such that for any distinct $(c,\ell),(d,m)\in \mathcal{C}\times \mathcal{L}$,  $\cos(x_{c,\ell},x_{d,m}) < \gamma_X$.

\section{Theoretical Curse of Multilinguality for Embedding Space Structure}
\label{sec:proofs}

\subsection{Defining the curse of multilinguality}
\label{sec:defining_com}

The \com is described in the field for empirical spaces as the degradation of desired properties -- usually, extrinsic task performance for some set of languages -- with an increasing number of encoded languages, given fixed model capacity.
For the \com with respect to embedding space structure (\comemb), ``desired properties'' correspond to perfect multilinguality, and ``capacity'' corresponds to dimensionality.
\textbf{We formulate the \comemb as the inability of a multilingual embedding space to maintain perfect multilinguality given an increasing number of encoded languages $L$ without a prohibitive increase in dimensionality.}

\subsection{Theorem}

In this section, we answer the following question: \emph{how much dimensionality does a perfect multilingual space accommodating $L$ languages theoretically require in terms of $L$?}
If the required dimensionality grows prohibitively with the number of languages, e.g. linearly or exponentially, then this can be termed a theoretical \comemb.

\begin{restatable}[]{thm}{thm:main-thm}
\label{thm:main-thm}
The minimum dimensionality of a perfect multilingual space encoding $L$ languages grows as $\Theta(\log L)$.
\end{restatable}

We prove this in two parts.
First, we upper-bound the minimum \emph{sufficient} dimensionality of a perfect multilingual space in terms of the number of encoded languages and the dimensionality of a fixed target concept space, in line with our setup in \Cref{sec:defining_conditions}.

\begin{restatable}[]{subthm}{thm:main-sufficient-dim}
\label{thm:main-sufficient-dim}
Let $D_Z$ be the dimensionality of the reference concept space $Z$.
Then there exists a perfect multilingual space $X^L$ encoding $L$ languages with minimum sufficient dimensionality $D_{X^L}=D_Z + O(\log L)$.
\end{restatable}

Next, we lower-bound the dimensionality of the space, i.e. show the \emph{necessary} dimensionality of a perfect multilingual space.
\begin{restatable}[]{subthm}{lmm:main-necessary-dim}
\label{thm:main-necessary-dim}
Any perfect multilingual space $X^L$ encoding $L$ languages must have dimensionality $D_{X^L}=\Omega(\log L)$.
\end{restatable}

Since $D_Z$ is independent of $L$, \Cref{thm:main-sufficient-dim} and \Cref{thm:main-necessary-dim} together establish \Cref{thm:main-thm}.

The following proofs make a fixed unit norm assumption for $Z$ and $X^L$. 
This is standard in the use of text embeddings for applications such as retrieval \citep{weller2026on}. 
Our proofs are also extendable to the case of variable norm as sketched in \Cref{app:norms}.

\subsection{Proof of \Cref{thm:main-sufficient-dim}}
\label{sec:proof_suff}
We proceed constructively.
Given $Z$ and $L$, we describe a multilingual space $X^L$ and show that (a) it maintains the multilinguality conditions, and that (b) its dimensionality follows the required bound.

\paragraph{Multilingual space construction}

Recall that for a concept $c$, we have its embedding $z_c$ in the concept space $Z$, with the resolution scale $\gamma$, and dimensionality $D_Z$.

We now define a space $U$ of \emph{language offsets}.
This space uses $D_L$ new dimensions, and is populated with $L$ language offset vectors $u_{\ell} \in \mathbb{R}^{D_L}$, with a fixed norm $ \|u_\ell\| = r$, and with a maximum similarity $\rho<1$ between any two language offsets. 
That is, we have $\cos(u_\ell,u_m) < \rho$ for all $\ell,m \in \mathcal{L}$ where $\ell\neq m$.

We then define our multilingual space $X^L$, by a direct sum construction of the concept space and the language offsets.
That is, $\forall c \in \mathcal{C}, \ell\in \mathcal{L},  x_{c,\ell}:= \frac{1}{\sqrt{1+r^2}} z_c \oplus u_\ell$, where $\oplus$ refers to vector concatenation.
As a result, $x_{c,\ell} \in X^L$ has dimensionality $D_{X^L} = D_Z + D_L$.
Note the norm term $\frac{1}{\sqrt{1+r^2}}$. 
Since $\|z_c\|=1$ and $\|u_\ell\|=r$, and $z_c$ and $u_\ell$ are orthogonal when embedded in $X^L$, we now have $X^L$ as a unit norm space with $\|x_{c,\ell}\|=1$ for all $c,\ell$ and dimensionality $D_{X^L}$.

\paragraph{Similarity bounds on offset vectors}
We construct the language offset vectors in the $D_L$-dimensional space $U$ with fixed norm $r$ such that their minimum and maximum cosine similarity lies in the following interval, guaranteed to be non-empty by a choosing $r \ge 1$ and a suitable choice of $\rho$:
\begin{equation}
\label{eq:language-offset-two-sided-cosine}
1-\frac{1-\gamma}{r^2} 
<
\cos(u_\ell,u_m)
<
\rho < 1
\qquad \text{for all } \ell\neq m.
\end{equation}

\paragraph{Similarity decomposition} 
Given that $X^L$ and $Z$ are unit-norm, we have the following relationship: $\cos(x_{c,\ell},x_{d, m}) = \langle x_{c,\ell},x_{d,m} \rangle$; analogously for $Z$.
Since $z_c$ and $u_\ell$ are orthogonal when embedded in $X^L$, we make the following useful inner-product decomposition for points in $X^L$, used throughout our proof:
\begin{align}
\langle x_{c,\ell},x_{d,m}\rangle
&=
\left\langle
\frac{z_c\oplus u_\ell}{\sqrt{1+r^2}},
\frac{z_{d}\oplus u_m}{\sqrt{1+r^2}}
\right\rangle =
\frac{1}{1+r^2}
\left\langle
z_c\oplus u_\ell,
z_{d}\oplus u_m
\right\rangle 
=
\frac{\langle z_c,z_{d}\rangle+\langle u_\ell,u_m\rangle}{1+r^2}
\label{eq:distance-decomposition}
\end{align}

\begin{restatable}[]{lmm}{lmm:X-monolingual-structure}
\label{lmm:X-monolingual-structure}
The space $X^L$ satisfies the monolingual structure condition (\Cref{dff:monolingual-structure}).
\end{restatable}
\vspace{-5mm}
\begin{proof}
We will show that for every language $\ell$, $X^L_\ell \sim_R Z$, i.e. each language subspace is identical to $Z$ under point similarity ranking.
Using \Cref{eq:distance-decomposition} for points in $X^L_\ell$, we have
\[
\begin{aligned}
\cos(x_{c,\ell},x_{c',\ell}) 
&= \langle x_{c,\ell},x_{c',\ell}\rangle 
= \frac{\langle z_c,z_{c'}\rangle+\langle u_\ell,u_\ell\rangle}{1+r^2} 
= \frac{\langle z_c,z_{c'}\rangle+r^2}{1+r^2} 
= \frac{\cos(z_c,z_{c'})+r^2}{1+r^2} \\
&= a\cos(z_c,z_{c'}) +b 
\end{aligned}
\]
for fixed and global $a$ and $b$ depending only $r$. 
Note that this automatically gives us, for any two languages $\ell,m$,
\[
\begin{aligned}
\cos(x_{c,\ell},x_{c',\ell})
&= a \cdot \cos(z_c,z_{c'}) + b 
= \cos(x_{c,m},x_{c',m})
\end{aligned}
\]

Given this affine behavior, and the language agnosticity of the result, it follows that the similarity ranking behavior is identical in all languages $\ell \in \mathcal{L}$ with respect to $Z$, i.e. $\forall \ell \in \mathcal{L}, X^L_\ell \sim_R Z$. Hence, $X^L$ satisfies the monolingual structure condition. 
\end{proof}

\begin{restatable}[]{lmm}{lmm:X-crosslingual-structure}
\label{lmm:X-crosslingual-structure}
The space $X^L$ satisfies the cross-lingual alignment condition (\Cref{dff:cross-lingual-alignment}).
\end{restatable}
\vspace*{-5mm}
\begin{proof}
We show that for every $c, c' \in \mathcal{C}$ where $c\neq c'$ and all languages $\ell,m,n \in \mathcal{L}$,
$\cos(x_{c,\ell},x_{c,m}) > \cos(x_{c,\ell},x_{c',n})$.
Recall that points in $Z$ have resolution scale $\gamma$ and points in $U$ have norm $r$.
From the similarity decomposition (\Cref{eq:distance-decomposition}), we can simplify and bound the two similarity expressions as follows for any $c \neq c', \ell, m, n$:

\[
\begin{aligned}
\cos(x_{c,\ell},x_{c,m})
&= \frac{\langle z_c,z_c\rangle+\langle u_\ell,u_m\rangle}{1+r^2} 
= \frac{1+\langle u_\ell,u_m\rangle}{1+r^2} \\
\cos(x_{c,\ell},x_{c',n})
&= \frac{\langle z_c,z_{c'}\rangle+\langle u_\ell,u_n\rangle}{1+r^2} 
< \frac{\gamma+r^2}{1+r^2}.
\end{aligned}
\]

Now it suffices to show that
\[
1+\langle u_\ell,u_m\rangle > \gamma+r^2.
\]

Since $\langle u_\ell,u_m\rangle=r^2\cos(u_\ell,u_m)$, the above expression is equivalent to
\[
\cos(u_\ell,u_m) > 1-\frac{1-\gamma}{r^2}.
\]
which is guaranteed by our construction in \Cref{eq:language-offset-two-sided-cosine}.
Thus, $X^L$ satisfies the cross-lingual alignment condition.
\end{proof}

\begin{restatable}[]{lmm}{lmm:X-non-degenerate}
\label{lmm:X-non-degenerate}
The space $X^L$ is non-denegerate.
\end{restatable}
\vspace{-5mm}
\begin{proof}
The cosine similarity between any two distinct points in $X^L$ is given by \Cref{eq:distance-decomposition}.
We bound the maximum similarity between any two points based on the three cases: same concept with different language, different concept with same language, and different concept with different language.
If $c=c'$ and $\ell\neq m$, then
$\cos(x_{c,\ell},x_{c,m}) < \frac{1+r^2\rho}{1+r^2}$.
If $c\neq c'$ and $\ell=m$, then
$\cos(x_{c,\ell},x_{c',\ell}) < \frac{\gamma+r^2}{1+r^2}$.
If both $c\neq c'$ and $\ell\neq m$, then
\[
\begin{aligned}
\cos(x_{c,\ell},x_{c',m})
&=
\frac{\langle z_c,z_{c'}\rangle+\langle u_\ell,u_m\rangle}{1+r^2}
<
\frac{\gamma+r^2\rho}{1+r^2} 
<
\gamma_{X}:=
\max\left\{
\frac{\gamma+r^2}{1+r^2},
\frac{1+r^2\rho}{1+r^2}
\right\}.
\end{aligned}
\]

Note that $\gamma_{X}<1$, since $\gamma<1$ and $\rho<1$.
$\gamma_{X}$ is independent of $L$ and only depends on the constants $\gamma$, $\rho$ and $r$; thus, $X^L$ maintains a fixed cosine scale independent of the number of languages encoded.
Thus, $X^L$ satisfies the non-degeneracy condition at fixed resolution scale $\gamma_X <1$.
\end{proof}

\begin{restatable}[]{lmm}{lmm:X-log-growth-sufficient}
\label{lmm:X-log-growth-sufficient}
The minimum sufficient dimensionality of $X^L$ is $O(\log L)$.
\end{restatable}
\vspace{-5mm}
\begin{proof}
Recall that $D_{X^L} = D_Z+D_L$.
We will bound $D_L$, the dimensionality of the language offset space $U$.

$U$ is constructed by placing $L$ fixed-norm language offsets in $U$ such that each point pair has cosine similarity in the non-empty interval defined in \Cref{eq:language-offset-two-sided-cosine}.
Bounding $D_L$ therefore reduces to the following question: how many dimensions $D_L$ are sufficient to pack $L$ points on a sphere in a $D_L$-dimensional with specified maximum and minimum cosine similarity? 

Using standard spherical cap packing bounds, we can show that the number of points that can be placed on a sphere in a $D_L$-dimensional space following these constraints grows exponentially in $D_L$.\footnote{This comes with the constraint on a choice of $\rho$. For $r=1$, this is $\arccos(\rho) < \arccos(\gamma)/2$.}
This is shown in detail in \Cref{app:packing_sufficient}.
Briefly, the maximum similarity constraint imposes an angular region within which all points must lie, and for a maximal packing, caps with angular radius equal to the minimum angular separation centered at the packed points must cover this containing region.
This allows us to lower-bound the number of points in a maximal packing by a ratio of spherical cap measures, which grows exponentially with the dimensionality of the sphere.
Therefore, to represent $L$ language offsets, it is possible to choose some constant $C>0$ independent of $L$ such that a sphere in a $D_L$-dimensional space with $D_L=\left\lceil \frac{1}{C}\log L\right\rceil$ fits $\exp(CD_L) \ge L$ such points.
Therefore, the minimum sufficient dimension to pack $L$ points maintaining the required constraints is $D_L=O(\log L)$. 
Since $D_Z$ is constant in terms of $L$, $X^L$ has the minimum sufficient dimension
\[
\begin{aligned}
D_{X^L}=D_Z+D_L =D_Z+O(\log L) \implies  D_{X^L} = O(\log L)
\end{aligned}
\]

Note that it suffices to exhibit a single perfect multilingual space with the above scaling to upper-bound the minimum sufficient dimensionality for perfect multilinguality.
By the existence of $X^L$, we establish  the minimum sufficient dimension for perfect multilinguality grows as $O(\log L)$.
\end{proof}

\paragraph{Fitting 7000 languages in 30 dimensions}
For practical purposes, the constant involved in the $O(\log L)$ bound may be of concern. 
In \Cref{app:constants}, we prove that it is theoretically possible to fit $L=7000$ language offset unit vectors into $D_L = 30$  dimensions while satisfying all required constraints on minimum and maximum similarity required for our constructive solution, with a realistic choice of $\gamma$.
Briefly, we show that a random construction of $L$ points in a $D_L$-dimensional space has a non-zero probability of satisfying our required constraints, proving existence.
Thus, our theoretical scaling factor is of minor practical concern given the typical dimensionality scale of representation spaces.

\subsection{Proof of \Cref{thm:main-necessary-dim}}
\label{sec:proof_nec}

\begin{proof}
Fix any concept $c$. Then the $L$ language variants of this concept, $x_{c,1},\ldots,x_{c,L}$,
are all distinct fixed-norm vectors, or distinct points on a sphere. 
Since $X^L$ is non-degenerate as we assume it is perfectly multilingual, any two distinct points in it must remain separated at some fixed cosine resolution scale $\gamma_X<1$, equivalent to a fixed angular separation $\theta_X := \arccos(\gamma_X)$, independent of $L$.\footnote{Fixing the resolution scale across the family of spaces $X^L$ allows us to study the asymptotic behaviour of the necessary dimensionality of $X^L$ with increasing $L$ in a meaningful manner, and follows the assumptions in related work studying dimensionality scaling \citep{weller2026on,okajima2026limits}.
}
Thus, $X^L$ contains at least $L$ points on a multidimensional sphere with a fixed angular separation, and its necessary dimensionality is lower bounded by the number of dimensions needed to place points in such a way.

We now apply a standard spherical cap bound. 
For any fixed angular separation $\theta_X>0$, the maximum number of vectors that can be placed on a $D_{X_L}$-dimensional sphere in with pairwise angular separation at least $\theta_X$ grows at most exponentially in $D_{X_L}$.
We show this in detail in \Cref{app:packing_necessary}.
Briefly, the minimum angular separation means that if we place a spherical cap with half the required angular separation around each point then these caps must be disjoint. 
Their total measure cannot exceed that of the sphere.
This upper-bounds the number of points by the reciprocal of a spherical cap measure, which grows at most exponentially with the dimensionality of the sphere.
That is, there exists a constant $A_{\theta_X}>0$, depending only on $\theta_X$, such that
\[
\begin{aligned}
L \le \exp(A_{\theta_X} D_{X_L}) &\Rightarrow \log L \le A_{\theta_X} D_{X_L}   
\Rightarrow D_{X_L} \ge \frac{1}{A_{\theta_X}}\log L 
&\Rightarrow D_{X_L}=\Omega(\log L)
\end{aligned}
\]
Thus, the dimensionality of any non-degenerate multilingual space encoding $L$ languages has a logarithmic dependence on $L$.
Note that if the original monolingual concept space ``fully'' requires $D_Z$ dimensions to represent concept semantics, then these $\Omega(\log L)$ dimensions required to construct language directions are additional to these $D_Z$ dimensions.
Under this minimality assumption on the base concept space, the additional cost can be understood as a multilinguality tax:\footnote{In general, real embedding spaces may not use their ambient dimensionality \citep{ansuini2019intrinsic}. Instead, this can be understood as placing pressure on true or effective dimensionality of such a space.}
\[
D_{X^L} = D_Z+\Omega(\log L)
\]
In general, the dimension $D_{X^L}$ of any perfect multilingual space $X^L$ encoding $L$ languages must grow at least as $\Omega(\log L)$.
\end{proof}

\section{Empirical curse of multilinguality for embedding spaces}
\label{sec:empirical_com}

While our theorem shows that a multilingual space is theoretically capable of exhibiting perfect multilinguality under language scaling without a prohibitive cost, there may still exist an empirical \comemb for real-world multilingual models.
In this section, we formulate metrics for monolingual structure and cross-lingual alignment,  and provide a small-scale controlled study showing the empirical curse of multilinguality for embedding space structure under eight configurations mimicking various real-world conditions.
Specifically, we train multilingual models from scratch on an increasing number of languages and evaluate embedding space structure under language scaling, in different training and evaluation configurations.\footnote{Our characterization of the theoretical curse 
of multilinguality for embedding space structure 
allows for a small growth based on number of languages, while in our experiments, we keep the dimensionality fixed with language scale for simplicity.
In practice, trained embedding spaces use very little of the ambient dimensionality \citep{ethayarajh-2019-contextual,ansuini2019intrinsic}, 
and the permissible numerical inflation for around a hundred languages is small as discussed in \Cref{app:constants}.}
We provide precise metric definitions and extra experimental details and results in \Cref{app:empirical_com}.

\paragraph{Note on previous work}
Previous works studying the \com and related concepts generally differ with respect to their experimental conditions and assumptions.
Differences include (i) whether token compute is fixed \citep{shaham-etal-2023-causes} or increases with more training languages \citep{aharoni-etal-2019-massivel}, (ii) whether the training data distribution is uniform over languages \citep{shaham-etal-2023-causes} or follows a real-world distribution \citep{aharoni-etal-2019-massivel,conneau-etal-2020-unsupervised}, and (iii) whether performance metrics are tracked over all training languages \citep{conneau-etal-2020-unsupervised} or a fixed group such as high-resource languages \citep{arivazhagan-etal-2019-massively}, among other variables.

\paragraph{Configurations} 
We study the empirical \comemb under eight configurations consisting of all combinations of the three above dimensions: (i) \texttt{\{fixed-compute, increasing-compute\}}, (ii) \texttt{\{uniform-sampling, realistic-sampling\}}, and (iii) \texttt{\{aggregate-all, aggregate-fixed\}}.

\paragraph{Metrics}

Our metrics are inspired from the multilinguality conditions.
We consider a ``concept'' to be a text, and compute the following using a fixed concept set across all target languages obtained from multiway parallel data.
We evaluate \textbf{monolingual structure} against a strong reference space, by measuring, for each concept, the overlap of the monolingual neighbourhood of that concept in each target language with the gold neighbourhood of that concept in a strong monolingual reference space. 
We report this metric (\emph{k-nearest-neighbour overlap; \texttt{MS-NNO}}) as a percentage of $k=20$, averaged over concepts and languages.
We also look at language mean over \emph{token-normalized masked language modeling loss} on a monolingual test set per language (\texttt{MS-MLM)} as an indirect measure of monolingual structure quality.
We evaluate \textbf{cross-lingual alignment} by checking for each language $\ell_1$, whether, for each concept and each language $\ell_2$, the $\ell_2$ equivalent of that concept is closer to the $\ell_1$ equivalent than other $\ell_2$ concepts (weak view; similar trends for strong view shown in \Cref{app:additional_results}). 
This metric (\texttt{CLA-WV}) is reported as a percentage of the concept set, averaged over language pairs. 
We also check for non-degeneracy by checking for looking at the minimum, maximum, and mean cosine similarity over all point pairs in the space as a function of increasing language coverage.
We do not see evidence of degeneracy in any configuration (\Cref{app:additional_results}).

\paragraph{Experimental setup}

We study a 4-layer Transformer-based \texttt{BERT} encoder architecture with a masked-language modeling training objective \citep{devlin2019bert}.
We use the \texttt{MADLAD-400} corpus \citep{kudugunta2023madlad} for monolingual training data, and \texttt{FLORES+} \citep{goyal-etal-2022-flores,nllb-24}, \texttt{BOUQuET} \citep{andrews-etal-2025-bouquet,omnilingual2026}, and \texttt{WMT24++} \citep{deutsch-2025-wmt} as multiway parallel data for evaluation.
We use the English concept space from \texttt{e5-large-v2} \citep{wang2022text} as a strong reference space for \texttt{MS-NNO}.
We create nine language groups with incrementally more languages added in descending order of resourcedness,\footnote{This is realistic in terms of how language coverage is expanded in the real world; we are also driven by data requirement constraints since initial smaller language groups require more data per language for some configurations.} with the largest language group containing 100 languages.
The token count per language depends on the token-compute and sampling settings: \texttt{fixed-compute} experiments are trained on a total budget of 500M tokens, and \texttt{increasing-compute} with \texttt{uniform-sampling} uses 50M tokens per language.
The \texttt{realistic-sampling} setting uses 500M as the maximum language token count, with other language token count proportionate to resource level in our training corpus.
Each model run consists of 5 seeds, with each seed sampling different data up to the budget, and random initialization of model.
The tokenizer is trained on a subset of the training corpus.
For the \texttt{aggregate-all} setting, we evaluate all training languages covered by each evaluation dataset.
For \texttt{aggregate-fixed}, the smallest language group is 10, and we evaluate these 10 fixed languages across all language groups.
We use the \texttt{[CLS]} token embedding as a text embedding for our trained models.

\paragraph{Results and discussion}

\begin{figure*}[t]
    \centering
    \includegraphics[scale=0.75]{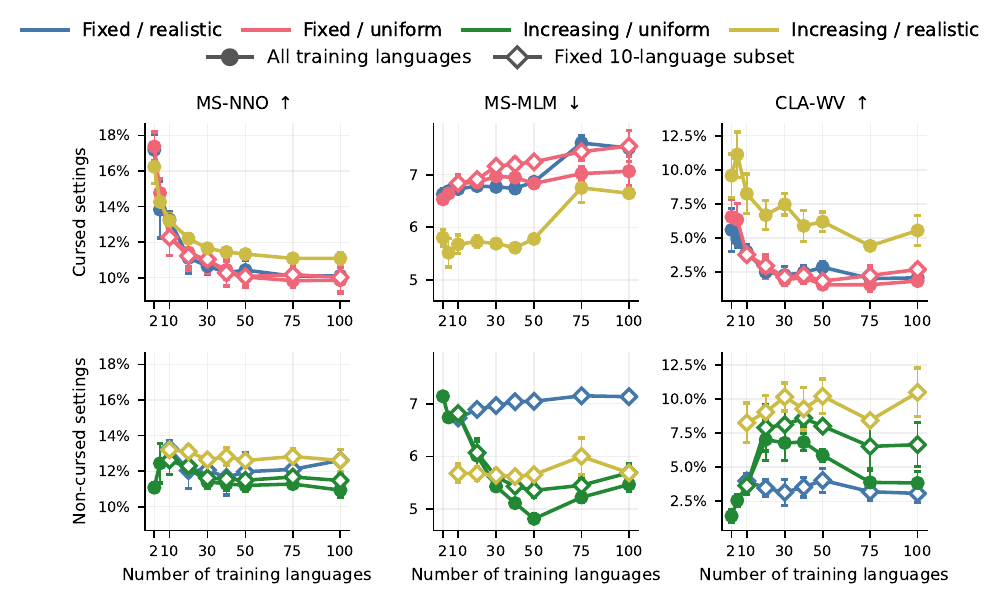}
    \caption{\texttt{MS-NNO}, \texttt{MS-MLM}, and \texttt{CLA-WV} on \texttt{BOUQuET} for encoder models with as the number of training languages increases, in eight configurations. The top and bottom rows show cursed and non-cursed configurations respectively. Arrows indicate whether higher or lower metric values are better. Error bars show standard deviation over five model seeds.}
    \label{fig:main-results-bouquet}
\end{figure*}

If all metrics exhibit deterioration for a configuration, we term it a ``cursed'' configuration.
\Cref{fig:main-results-bouquet} shows results for all eight configurations with \texttt{BOUQuET}, separating cursed from non-cursed configurations  (similar trends for other evaluation datasets in \Cref{app:additional_results})
We observe four cursed configurations of eight, showing consistent degradation of monolingual structure and cross-lingual alignment. 
\textbf{Therefore, we observe an empirical curse of multilinguality for embedding space structure, parallelling previous work in the \com.}

In the other four non-cursed configurations, we observe stability or improvements in one or more metrics.
We find that these are the configurations that generally maintain high token count for the target languages of interest as a consequence of compute and sampling choices.
For example, with \texttt{increasing-compute}, \texttt{realistic-sampling}, and \texttt{aggregate-fixed} (\tikz[baseline=-0.5ex]{
  \draw[yellow!80!orange, line width=1.8pt] (0,0) -- (1.2em,0);
  \draw[yellow!80!orange, fill=white, line width=1.2pt]
    (0.6em,0.25em) -- (0.88em,0) -- (0.6em,-0.25em) -- (0.32em,0) -- cycle;
}), token count for the 10 high-resource target languages remains high and unchanged with more training languages, and metrics stay roughly stable. 
We also see that metrics may actually show some improvements with more languages up to a point, e.g. with \texttt{increasing-compute} and \texttt{uniform-sampling} (\raisebox{0.4ex}{\textcolor{green!60!black}{\rule{1em}{1.2pt}}}), which indicates that total token count and cross-lingual transfer may also play a role. 
\textbf{In general, our findings suggest that the empirical curse of multilinguality for embedding space structure is sensitive to training and evaluation configurations and is mitigated in favourable conditions.}

\paragraph{Recommendation regarding empirical studies}
The above finding regarding the \comemb parallels the fact that various previous empirical studies of the extrinsic \com differ in their characterization of the phenomenon. 
These often operate in different configurations based on their application of interest, overloading the catchall term ``curse of multilinguality'' and related concepts.
Given our findings, we recommend that future work make their configuration of interest explicit to promote scientific clarity regarding the \com.

\section{Conclusion}
\label{sec:conclusion}
The curse of multilinguality describes the phenomenon of performance degradation in fixed capacity multilingual models with increasing language scale.
Our work provides the first theoretical and intrinsic grounding to this phenomenon.
Specifically, we look at whether multilingual representation spaces can theoretically can be scaled in language coverage while maintaining their quality.
We formulate the notion of \emph{perfect multilinguality} for embedding spaces, embodied by two \emph{multilinguality conditions}.
Our main contribution is a theoretical result that shows it is necessary and sufficient to scale representation space capacity only logarithmically in the number of languages while maintaining perfect multilinguality.
This shows that there is no theoretical curse of multilinguality for embedding space structure.
We also demonstrate the empirical curse of multilinguality for embedding space structure for the first time, and show that it is sensitive to data, compute, and evaluation conditions.
Our findings contribute to better formulating and understanding the curse of multilinguality, and raise questions regarding bridging the gap between theory and practice.

\section*{Limitations}

\paragraph{Perfect multilinguality} Our formulation of perfect multilinguality for multilingual representation spaces is most directly relevant for applications such as embedding-based retrieval.
Different multilingual representation spaces, such as intermediate representations of decoder models, may require adaptation based on needs from different downstream application scenarios.
Further, this conception uses a simplified scenario consisting of a global concept set with global semantics across languages. 
This allows the cost of concept semantics to stay constant as language coverage grows in theory.
In practice, different languages may introduce language-specific concepts, or have differing semantic relationships between concepts, introducing language-specific needs and additional dimensionality costs.
We leave this to future work to incorporate into a theoretical model.

\paragraph{Model scale} Our empirical studies are for small encoder models given the controlled nature and language scale of our study.
We leave it to future work to study the impact of model scaling and different architectures on the curse of multilinguality for embedding space structure.

\subsection*{AI use statement}
In this work, we used generative AI tools as a soundboard for mathematical ideas, for ironing out details of and checking for errors in proofs, and for help with ideation for the proof in \Cref{app:constants}.
All mathematical arguments in this paper have been written and verified by the authors. 
We used AI-assisted literature search tools in combination with manual search for discovering previous work, following by a manual review of relevance.
We used coding assistants for our experiments, with human review. 
AI was used for typesetting of mathematical expressions and formatting scientific figures in the writing of this paper, and for minor writing refinements.
We take responsibility for the final content of this work, including text, claims or artifacts produced with the aid of generative AI.

\bibliography{iclr2027_conference}
\bibliographystyle{iclr2027_conference}

\clearpage

\appendix
\onecolumn
\section{Proof details}
\label[appendix]{app:proofs}

\subsection{Spherical cap bound for minimum sufficient dimensionality}
\label[appendix]{app:packing_sufficient}

In this section, we describe and prove the result necessary for the proof in \Cref{sec:proof_suff}.
Specifically, we will show that the number of points that can be placed on a $D_L$-dimensional sphere with radius $r$ given a minimum and maximum cosine similarity following \Cref{eq:language-offset-two-sided-cosine}, and given a suitable choice of $\rho$, grows exponentially in $D_L$.

Recall that $\gamma<1$ is the maximum allowed similarity in the non-degenerate space $Z$, given to us. 
We can assume that $\gamma>0$; if observed maximum cosine similarity is negative, the non-degeneracy condition would still hold for a larger $\gamma$.

We will use $r = 1$ in our construction for simplicity and without loss of generality.
\Cref{eq:language-offset-two-sided-cosine} simplifies to 
\[
 \gamma < \cos(u_\ell, u_m) < \rho < 1
\]
We will choose $\rho$ such that  $\gamma<\rho<1$ and $\arccos(\rho) < \arccos(\gamma)/2$.

Now, we will show the following.

\begin{lmm}
Fix constants $0<\gamma<\rho<1$, such that $\arccos(\rho) < \arccos(\gamma)/2$. Then $d$ dimensions are sufficient to pack $\exp(\Omega(d))$ points on the unit sphere $S^{d-1}$ with maximum cosine similarity $\rho$ and minimum cosine similarity $\gamma$. That is, there exists a set $\mathcal{M} \subset S^{d-1}$ such that for all distinct $u,u'\in \mathcal{M}$, $\gamma < \langle u,u'\rangle < \rho$, and $|\mathcal{M}|=\exp(\Omega(d))$.
\end{lmm}

\paragraph{Proof sketch}
We will place points on a cap on a multi-dimensional sphere bounded with some angular radius.
Since all points are within this cap, the maximum cosine similarity is bounded, letting us satisfy the maximum similarity constraint. 
Now we will fit points on this cap such that any two points have a minimum angular radius from each other. 
Imagine each such point making a small cap around itself with angular radius as the minimum required angular separation.
Now, consider the maximal set of fitting points in the big cap.
The union of the area of the small caps must cover the area of the big cap, since, if there was an uncovered point, that would mean that it was not within the minimum required separation of any other point, and we would be able to add in to our maximal set.
Since we can express the union of the area of the small caps in terms of the number of points, we therefore get a lower bound on the maximum number of fitting points in terms of the areas of the two caps.
By standard spherical cap results, the area of a spherical cap grows exponentially in the dimension of the sphere.
Thus, as long as the angular radius of the big cap is large enough relative to that of the small cap, guaranteed by our choice of $\rho$, the maximum number of points we can fit that satisfy the above constraints also grow exponentially in the dimension of the sphere, giving us the required result that the sufficient dimension to fit $L$ points is logarithmic in $L$.

\begin{proof}
Let $\theta_{\min}:=\arccos(\rho)$ and $\theta_{\max}:=\arccos(\gamma)$. 
Since $\arccos(\rho) < \arccos(\gamma)/2$, and $\gamma > 0$, we have $ \theta_{\min}<\theta_{\max}/2 < \pi/4$.

Choose $\beta$ such that $\theta_{\min}<\beta<\frac{\theta_{\max}}{2} < \pi/4$.
Fix a direction $v\in S^{d-1}$ and define the spherical cap around $v$:
\[
\begin{aligned}
& \mathcal C_\beta(v)  := \{u\in S^{d-1}: \angle(u,v) < \beta\}.
\end{aligned}
\]

This means that any two points in this cap $u$ and $u'$ can have angular distance at most $2\beta<\theta_{\max}$ (since each has angular distance at most $\beta$ from $v$).

Now consider a set of points $\mathcal{Q}$ inside $\mathcal C_\beta(v)$ such that every pair in this set has angular distance greater than $\theta_{\min}$. Equivalently, $\langle u,u'\rangle < \rho$. 
Thus, for any distinct $u,u'\in \mathcal{Q}$, we will have
\[
\begin{aligned}
& \theta_{\min} < \angle(u,u') < 2\beta<\theta_{\max} \\
& \implies
\cos(\theta_{\max}) 
<
\langle u,u'\rangle
<
\cos(\theta_{\min}) \\
& \implies \gamma < \langle u,u'\rangle < \rho
\end{aligned}
\]

Thus, the set $\mathcal{Q}$ satisfies our required angular constraints.
In order to prove our lemma, it is enough to show that the maximal size of $\mathcal{Q}$ grows exponentially with $d$; in other words, that $d$ dimensions are enough to pack some set of points $\mathcal{M}$ with size growing exponentially in $d$.

Let $\mathcal{M}=\{u_1,\ldots,u_M\}\subset \mathcal C_\beta(v)$ with size $M = |\mathcal{M}|$ be a maximal set satisfying $\angle(u_i,u_j) > \theta_{\min} ~\text{for all } i\neq j$.

Now, any point in $\mathcal C_\beta(v)$ must be within angular distance $\theta_{\min}$ of some point in $\mathcal{M}$; otherwise this point could be added to $\mathcal{M}$. 
Therefore, $C_\beta(v)$ is a subset of the union of spherical caps created by defining a minimum angle radius around each point in $\mathcal{M}$. In other words: 
\[
\mathcal C_\beta(v)
\subseteq
\bigcup_{i=1}^M \mathcal C_{\theta_{\min}}(u_i).
\]
Now we will consider the surface area on the sphere. 
Taking normalized spherical measure $\sigma$, and since $\sigma$ only depends on angular radius, 
\[
\begin{aligned}
\sigma(\mathcal C_\beta(v))
& \le
\sum_{i=1}^M \sigma(\mathcal C_{\theta_{\min}}(u_i))
=
M\sigma(\mathcal C_{\theta_{\min}}) \\
& \implies
M
\ge
\frac{\sigma(\mathcal C_\beta)}{\sigma(\mathcal C_{\theta_{\min}})}.
\end{aligned}
\]

For fixed angular radii, spherical cap measure scales exponentially in $d$ as per standard cap measure results \citep{naszodi2016some}. 
Specifically, for $0<\alpha<\pi/2$, the normalized measure of a spherical cap of angular radius $\alpha$ in $S^{d-1}$ given by \citet[Lemma~3.4, Eq.~(7)]{naszodi2016some}, attributed to \citet{boroczky2003covering}, satisfies
\[
\sigma(\mathcal C_\alpha)
=
\Theta_\alpha\!\left(
\frac{\sin^{d-1}(\alpha)}{\sqrt d}
\right)
\]%

Applying this with $\alpha=\beta < \pi/2$ and
$\alpha=\theta_{\min}<\pi/2$ gives
\[
M
\ge
\frac{\sigma(\mathcal C_\beta)}
     {\sigma(\mathcal C_{\theta_{\min}})}
=
\Omega\!\left(
\left(
\frac{\sin\beta}
     {\sin\theta_{\min}}
\right)^{d-1}
\right),
\]

Since $\beta>\theta_{\min}$, we have $\sin\beta>\sin\theta_{\min}$. Since $\beta$ and $\theta_{\min}$ are fixed, there exists a constant $C_M>0$, independent of $d$, such that
\[
\begin{aligned}
M\ge \exp(C_Md) \implies M = \exp(\Omega(d))
\end{aligned}
\]
That is, there exists some set $\mathcal{M}$ containing $\exp(\Omega(d))$ points whose pairwise cosine similarities lie between $\gamma$ and $\rho$, on a $d$-dimensional sphere.

In other words, the minimum sufficient dimension to fit $L$ points whose pairwise cosine similarities lie between $\gamma$ and $\rho$ on $S^{d-1}\subset\mathbb R^d$ is $d=O(\log L)$.
\end{proof}

\subsection{Spherical cap bound for necessary dimensionality}
\label[appendix]{app:packing_necessary}

In this section, we describe and prove the result necessary for the proof in \Cref{sec:proof_nec}.
Specifically, we will show that the dimensionality required to place $L$ fixed-norm points on a $D_{X^L}$-dimensional sphere with some minimum angular separation grows at least logarithmically in $L$.

Given $\gamma < 1$ as the maximum allowed cosine similarity, we will use minimum angular separation $\theta = \arccos (\gamma)$, with $0<\theta<\pi$. 
Now, we will show the following:

\begin{lmm}
Fix a constant $0<\theta<\pi$. 
Suppose $\mathcal{M}\subset S^{d-1}$ is a set of $M=|\mathcal{M}|$ points on a unit sphere with pairwise angular separation greater than $\theta$. 
That is, for all distinct $u,u'\in \mathcal{M}$, $\angle(u,u')> \theta$. 
Then there exists a constant $A_\theta>0$, independent of $M$ and $d$, such that $M\le\exp(A_\theta d)$.
Equivalently, any such packing of $M$ points requires $d=\Omega(\log M)$.
\end{lmm}

\paragraph{Proof sketch}
We will place $M$ points on a multidimensional sphere.
Think of each point in $\mathcal{M}$ forming a spherical cap on the sphere with angular radius as half of the minimum required angular separation.
Since each pair of points must have the required angular separation, these spherical caps must be disjoint: if a point on the sphere belonged to caps associated with two points in $\mathcal{M}$, those two points would not satisfy the criterion by triangle inequality. 
Thus, the surface area of the whole sphere must be larger than the union of the surface area of the spherical caps.
Since we have $M$ spherical caps and can express the union of the surface area of the spherical caps in terms of $M$, we get an upper bound on $M$ in terms of an inverse spherical cap measure, which scales exponentially in the dimension of the sphere.
Thus, we get that the minimum dimension that can fit $M$ points is at least logarithmic in $M$.

\begin{proof}
Let $\mathcal{M}=\{u_1,\ldots,u_M\}\subset S^{d-1}$ satisfy $\angle(u_i,u_j)> \theta \qquad \text{for all } i\neq j$.

Around each point $u_i$, place a spherical cap of angular radius $\theta/2$:
\[
\mathcal C_{\theta/2}(u_i)
:=
\{u\in S^{d-1}: \angle(u,u_i)\le \theta/2\}.
\]
These caps are disjoint, since if some point $w$ belonged to both $\mathcal C_{\theta/2}(u_i)$ and $\mathcal C_{\theta/2}(u_j)$, then by the triangle inequality, $\angle(u_i,u_j) \le \angle(u_i,w)+\angle(w,u_j) \le \theta$,
contradicting the separation constraint.

Taking normalized spherical measure $\sigma$, we have
\[
\sum_{i=1}^M \sigma(\mathcal C_{\theta/2}(u_i))
\le
\sigma(S^{d-1})
=
1
\]
Since $\sigma$ only depends on angular radius,
\[
M\sigma(\mathcal C_{\theta/2})
\le
1
\implies
M
\le
\frac{1}{\sigma(\mathcal C_{\theta/2})}.
\]

For fixed angular radius, spherical cap measures are at least exponentially small in $d$. More concretely, for $0<\alpha<\pi/2$, the normalized measure of a spherical cap of angular radius $\alpha$ in $S^{d-1}$, as given by \citet[Lemma~3.4, Eq.~(7)]{naszodi2016some}, attributed to \citet{boroczky2003covering} satisfies
\[
\sigma(\mathcal C_\alpha)
\ge
\frac{1}{\sqrt{2\pi d}}\sin^{d-1}(\alpha).
\]
Applying this with $\alpha=\theta/2<\pi/2$ gives
\[
\begin{aligned}
M
& \le
\frac{1}{\sigma(\mathcal C_{\theta/2})}
\le
\sqrt{2\pi d}
\left(
\frac{1}{\sin(\theta/2)}
\right)^{d-1} \\
& \implies
\log M
\le
\frac12\log(2\pi d) + (d-1)\log\!\left(\frac{1}{\sin(\theta/2)}\right)
\end{aligned}
\]

Since $\theta$ is fixed, $\sin(\theta/2)$ is a fixed constant in $(0,1)$. 
Since $\log(2\pi d) = o(d)$, there exists a constant $A_\theta>0$ depending only on $\theta$, such that

\[
\begin{aligned}
\log M \le A_\theta d  \implies  d \ge \frac{1}{A_\theta}\log M \implies d=\Omega(\log M)
\end{aligned}
\]

Therefore any set of $L$ points on $S^{d-1} \subset \mathbb{R}^d $ with fixed pairwise cosine separation $\gamma$ requires $d=\Omega(\log L)$.
\end{proof}

\subsection{Fitting 7000 languages in 30 extra dimensions}
\label[appendix]{app:constants}

Our constructive proof in \Cref{sec:proofs} and \Cref{app:packing_sufficient} constructs a language offset space $U$ with dimensionality $D_L$, and shows that that $D_L=O(\log L)$ dimensions are sufficient to fit $L$ points on a $D_L$-dimensional unit sphere given minimum cosine similarity $\gamma$ and maximum cosine similarity $\rho$.

However, this bounds hides a constant which may be of concern for practical considerations.
In this section, we prove the following:
\begin{lmm}
It is possible to fit $L = 7000$ points on a $30$-dimensional sphere with minimum pairwise cosine similarity $\gamma = 0.9$ and maximum cosine similarity $\rho= 0.99$.
\end{lmm}

$\gamma$ is set based on the choice of $0.1$ as a standard scale of cosine separation for text embedding models \citep{weller2026on}.
Note that our proof requires $\arccos(\rho) < \arccos(\gamma)/2$ which holds for this choice of $\rho$.

\paragraph{Sketch of proof}
We will construct each point having 1 shared dimension with an appropriate length, which will guarantee minimum similarity. 
We will use the remaining 29 dimensions to guarantee difference, i.e. ensure that any two points have atmost the maximum similarity. 
In particular, we will prove this statement using an existence argument: if we choose 7000 points at random on a sphere in 29 dimensions, the probability that no pair of points crosses a particular threshold of similarity is non-zero.
The existence of such a configuration gives us the above result.

\begin{proof}
Consider the space $\mathbb{R}^{30} = \text{span}(e_0) \oplus \mathbb{R}^{29}$, where $e_0$ is a unit vector orthogonal to $\mathbb{R}^{29}$. 
Now, construct
\[
x_i
=
\sqrt{0.945}\,e_0
+
\sqrt{0.055}\,y_i,
\qquad
y_i\in S^{28}\subset\mathbb{R}^{29}.
\]
Then \(\lVert x_i\rVert_2=1\) and since all vectors have unit norm, cosine similarity is equal to the inner product. Now, we consider the similarity between any two vectors $x_i$ and $x_j$:
\[
\begin{aligned}
\langle x_i,x_j\rangle
&=
0.945+0.055\langle y_i,y_j\rangle.
\end{aligned}
\]

Therefore, if
$\left|\langle y_i,y_j\rangle\right| \leq \frac{9}{11}$,  $\text{for all }i\neq j$,
then
\[
\begin{aligned}
\langle x_i,x_j\rangle
&\in
\left[
0.945-0.055\cdot\frac{9}{11},
\,
0.945+0.055\cdot\frac{9}{11}
\right] \\
&=
[0.9,0.99].
\end{aligned}
\]

It remains to show that such \(y_1,\ldots,y_{7000}\) exist.
Suppose that
\[
y_1,\ldots,y_{7000}
\overset{\mathrm{i.i.d.}}{\sim}
\operatorname{Unif}(S^{28}).
\]

Fix any pair \(i\neq j\). Since the distribution on the sphere is rotationally invariant, we may fix one point \(Y=e_1 = [1,0,0,...]\) and choose only the second point \(Z\sim\operatorname{Unif}(S^{28})\). The distribution of the inner product is unchanged. Thus, the dot product of $Y$ and $Z$ is distributed identically to just the first coordinate of $Z$:
\[
\langle Y,Z\rangle
\stackrel{d}{=}
Z_1,
\]
$Z_1$ is the first coordinate of a uniformly random point on \(S^{28}\). The distribution of the coordinate's value can be expressed in terms of a Beta distribution, as follows,\footnote{Briefly, this is because of the following: we generate a random point by choosing $29$ dimensions independently and normalizing all coordinates to result in unit norm. The first coordinate can then be expressed as $Z_1^2 \sim \frac{A}{A+B}$ where $A \sim \chi^2_1$ and $B \sim \chi^2_{28}$ and $A$ and $B$ are independent. This standardly gives us $Z_1^2 \sim \operatorname{Beta}\!\left(\frac12,14\right)$.}
\[
\langle Y,Z\rangle^2
\sim
\operatorname{Beta}\!\left(\frac12,14\right).
\]

Therefore, for any fixed pair,
\begin{align}
p
&:=
\Pr\left(
\left|\langle Y,Z\rangle\right|^2
>
\left(\frac{9}{11}\right)^2
\right) \\
&=
I_{40/121}\!\left(14,\frac12\right)
\approx
3.3465\times10^{-8},
\end{align}
where \(I_x(a,b)\) denotes the regularized incomplete beta function.

There are \(\binom{7000}{2}\) pairs. By the union bound,
\[
\begin{aligned}
\Pr\left(
\exists\,i<j:
\left|\langle y_i,y_j\rangle\right|
>
\frac{9}{11}
\right)
&\le
\binom{7000}{2}p \\
&\approx
0.8198
<
1.
\end{aligned}
\]
Hence,
\[
\Pr\left(
\left|\langle y_i,y_j\rangle\right|
\le
\frac{9}{11}
\text{ for all }i\neq j
\right)
>
0.
\]
Therefore, there exists a realization \(y_1,\ldots,y_{7000}\) satisfying the required bound, and the corresponding vectors \(x_1,\ldots,x_{7000}\in\mathbb{R}^{30}\) have pairwise cosine similarities in \([0.9,0.99]\).
That is, it is possible to fit $7000$ points, or language offsets, satisfying the constraints from our proof such that the resulting space would maintain the multilinguality conditions, in $30$ additional dimensions.
\end{proof}

\subsection{Fixed versus variable norm assumptions}
\label[appendix]{app:norms}

Our proof of \Cref{thm:main-thm} assumes a fixed-norm concept space $Z$ and multilingual space $X^L$.
This proof is extendable to the variable norm case.
We provide a brief sketch of the extension below.

\paragraph{Extending \Cref{thm:main-sufficient-dim}} 
We will still use fixed-norm $U$ in our construction, assuming variable-norm $Z$, resulting in variable-norm $X^L$ constructed as $x_{c,\ell} = z_c \oplus \|z_c\| u_\ell$.
Since $\|u_\ell\|=r$, we have $\|x_{c,\ell}\| = \|z_c\|\sqrt{1+r^2}$, and therefore the similarity decomposition in \Cref{eq:distance-decomposition} still works out as
\[
\begin{aligned}
\cos(x_{c,\ell},x_{d,m})
&=
\frac{
\langle z_c,z_d\rangle
+
\|z_c\|\cdot \|z_d\|\langle u_\ell,u_m\rangle
}{
\|z_c\|\cdot \|z_d\|(1+r^2)
}\\
&=
\frac{
\cos(z_c,z_d)+\langle u_\ell,u_m\rangle
}{
1+r^2
}
\end{aligned}
\]
Now, our proofs showing that $X^L$ upholds the multilinguality conditions and is non-degenerate follow as before.
The proof of the upper-bound on the sufficient dimensionality of $U$ in \Cref{app:packing_sufficient} also works as-is, since we used fixed-norm $U$.

\paragraph{Extending \Cref{thm:main-necessary-dim}} 
Given that $X^L$ is variable-norm, we can simply normalize each point in it and obtain unit-norm space $\widehat{X}^L$. 
This operation does not change the dimension of the space and maintains cosine similarities.
Now, we can lower-bound the dimension of $\widehat{X}^L$ exactly as in \Cref{sec:proof_nec} and this bound holds for variable norm $X^L$.

\section{Additional details and results for the empirical curse}
\label[appendix]{app:empirical_com}

\subsection{Metrics for multilinguality conditions}
\label[appendix]{sec:metrics}

We provide more details here on metrics used in \Cref{sec:empirical_com}.
We use multiparallel data, and so the concept set is identical across languages.

\paragraph{Monolingual structure} 
As per \Cref{dff:monolingual-structure}, the monolingual structure condition is concerned with the quality of monolingual relationships for a given language, and uses a theoretical ideal concept space as reference.
Given a reference space $Z$ and target space $X$, we use \emph{k-nearest-neighbour overlap (\texttt{MS-NNO})} to measure the quality of monolingual neighbourhoods in the subspace $X_t \subset X$, for each language $t$. 
Given $\mathcal{T}$ is the set of target languages, $\mathcal{C}$ is the set of language-agnostic concepts, $\mathcal{N}_X^t(c,k)$ is the neighbourhood of $k$ closest concepts for $c$ in language $t$ in space $X$, and $\mathcal{N}_Z(c,k)$ is the reference concept neighbourhood,
\[
\begin{aligned}
\mathrm{NNO}(X,Z,k)
&= \frac{1}{|\mathcal{T}|\cdot|\mathcal{C}|\cdot k}
   \sum_{t \in \mathcal{T}}\sum_{c \in \mathcal{C}} \sum_{n^t \in \mathcal{N}_Z(c,k)}
   \mathbf{1}[n^t \in \mathcal{N}_X^t(c,k)]
\end{aligned}
\]

We consider the English concept space from a strong reference model as our theoretical ideal space $Z$ and use mappings from multiparallel data to check set membership as in above.

Note that monolingual text embedding quality in the field of retrieval is measured by \texttt{recall@k} among other metrics, given query and relevance judgments. 
Our evaluation is an adjustment given the lack of multiparallel retrieval datasets and relevance judgments at our required language scale, and therefore treating monolingual neighbourhoods from a strong model as the gold target set.

We use $k=20$ for reporting results, and observe similar trends for varying $k=5,10,20,50$.

We also look at language mean over \emph{normalized masked language modeling loss} on a monolingual test set per language (\texttt{MS-MLM)} as an indirect measure of monolingual structure quality.

\paragraph{Cross-lingual alignment}

Drawing from \Cref{dff:cross-lingual-alignment}, and similar to previous work \citep{roy-etal-2020-lareqa,hammerl-etal-2024-understanding}, we measure cross-lingual alignment in embedding space structure by checking for each language $\ell_1$, whether, for each concept and each language $\ell_2$, the $\ell_2$ equivalent of that concept is closer to the $\ell_1$ equivalent than other $\ell_2$ concepts (weak view; \texttt{CLA-WV}) or any any other concepts regardless of language (strong view; \texttt{CLA-SV}).
Given the set of target languages $\mathcal{T}$, $\ell_s, \ell_t \in \mathcal{T}$, and $\mathcal{G}$ as the comparison group depending on strong or weak view,
\[
\begin{aligned}
\mathrm{CLA}
&= \frac{1}{|\mathcal{C}|\cdot|\mathcal{T}|(|\mathcal{T}|-1)}
   \sum_{\ell_s}\sum_{\ell_t\neq\ell_s}\sum_c \mathbf{1}\!\left[
   \mathrm{argmax}_{x' \in G}(\cos(x_{c,\ell_s}, x')) = x_{c,\ell_t}
   \right].
\end{aligned}
\]
where $\mathcal{G} = \{x_{c', \ell_t} | c'\neq c \}  \cup \{x_{c,\ell_t}\}$ in the weak view and $\mathcal{G} = \{x_{c', \ell_m} | c'\neq c, \ell_m \in \mathcal{T} \}  \cup \{x_{c,\ell_t}\}$ for the strong view.
We study both the weak view and strong view in our experiments and find similar trends.
We report the former.
Note that this evaluation is equivalent to measuring performance on a parallel text mining task.

\subsection{Further experimental details}

\paragraph{Models} 
For the reference model used for monolingual structure metric, we use a strong English text embedding model: \texttt{e5-large-v2}\footnote{\url{huggingface.co/intfloat/e5-large-v2}} \citep{wang2022text}, and \texttt{Llama-3.1-8B-Instruct}\footnote{\url{huggingface.co/meta-llama/Llama-3.1-8B-Instruct}} \citep{grattafiori2024llama} with last token pooling.
We obtain consistent trends. 
Reported results use the former.

\paragraph{Languages} 
Training languages were chosen in priority of token availability in the \texttt{MADLAD-400} corpus, as well as availability in evaluation datasets.
We evaluate on all languages covered in the training language group: \texttt{FLORES+} (90 languages) \citep{goyal-etal-2022-flores,nllb-24}, \texttt{BOUQuET} (90 languages) \citep{andrews-etal-2025-bouquet,omnilingual2026}, and \texttt{WMT24++} (55 languages) \citep{deutsch-2025-wmt}.

\subsection{Additional results}
\label[appendix]{app:additional_results}

\paragraph{Other datasets}

We show results for two other multiparallel datasets in \Cref{fig:main-results-flores} and \Cref{fig:main-results-wmt}. 
We also show \texttt{CLA-SV}, which shows similar trends for \texttt{WMT24++}, but is generally too low for \texttt{FLORES+} and \texttt{BOUQuET} to show meaningful trends.

These show similar trends as shown in \Cref{fig:main-results-bouquet}.

\begin{figure*}[!h]
    \centering
    \includegraphics[scale=0.7]{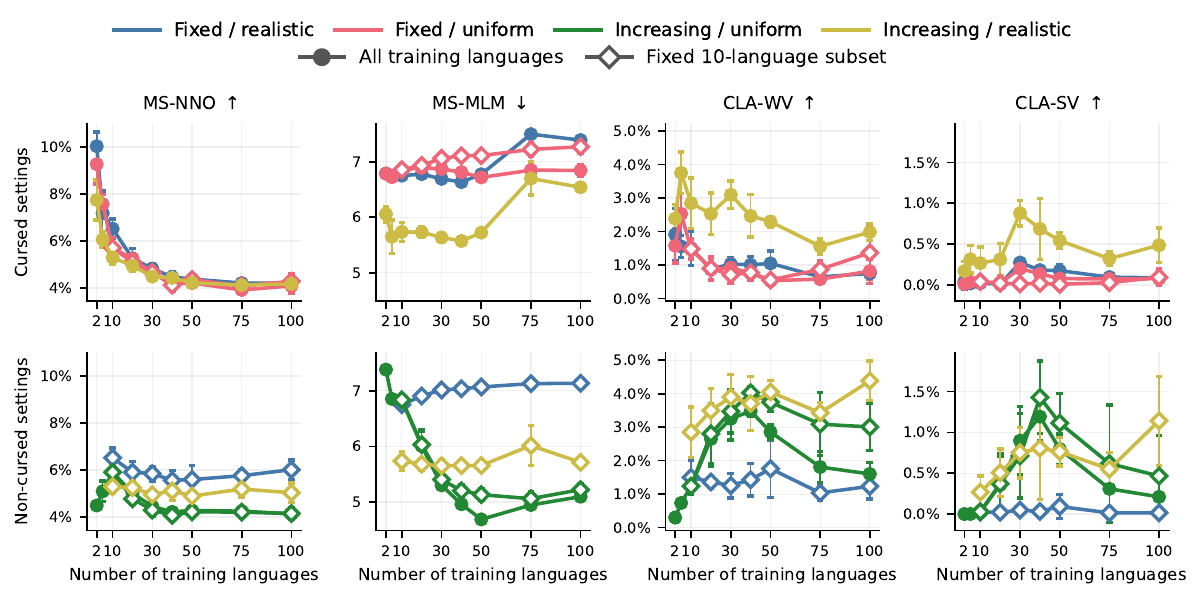}
    \caption{\texttt{MS-NNO}, \texttt{MS-MLM}, and \texttt{CLA} on \texttt{FLORES+} for encoder models with as the number of training languages increases, in eight configurations. The top and bottom rows show cursed and non-cursed configurations respectively. Arrows indicate whether higher or lower metric values are better. Error bars show standard deviation over five model seeds.}
    \label{fig:main-results-flores}
\end{figure*}

\begin{figure*}[!h]
    \centering
    \includegraphics[scale=0.7]{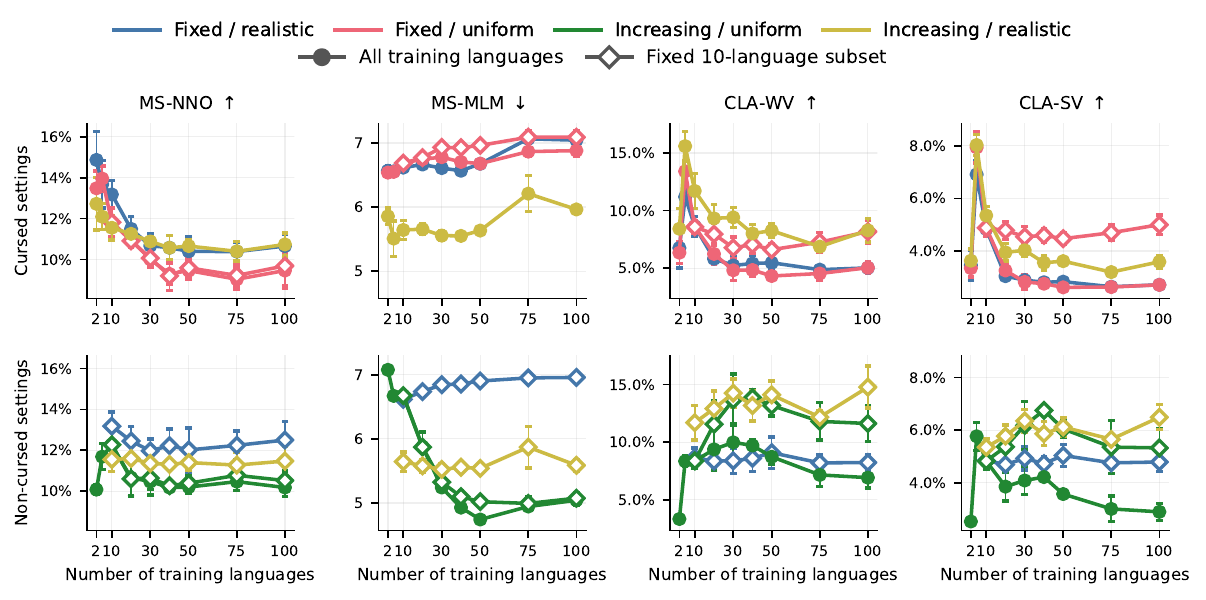}
    \caption{\texttt{MS-NNO}, \texttt{MS-MLM}, and \texttt{CLA} on \texttt{WMT24++} for encoder models with as the number of training languages increases, in eight configurations. The top and bottom rows show cursed and non-cursed configurations respectively. Arrows indicate whether higher or lower metric values are better. Error bars show standard deviation over five model seeds.}
    \label{fig:main-results-wmt}
\end{figure*}

\paragraph{Non-degeneracy}
We plot the mean and standard deviation of pairwise cosine similarity across models for \texttt{BOUQuET} in \Cref{fig:noncollapse-bouquet} (similar for other datasets).
This is stable, although generally high, which is evidence of high anisotropy in the space.

\begin{figure}[!h]
    \centering
    \includegraphics[scale=0.7]{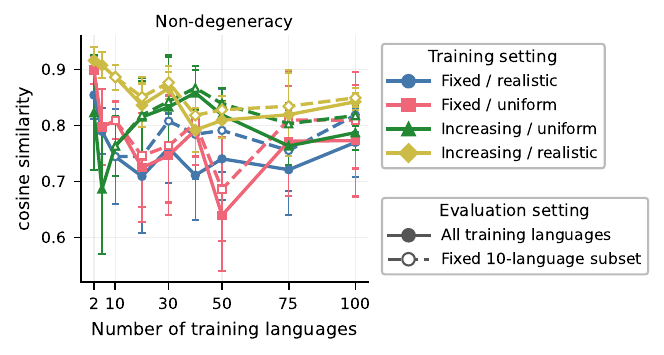}
    \caption{Mean pairwise similarity on \texttt{BOUQuET} across training and evaluation settings. Error bars show standard deviation over five seeds.}
    \label{fig:noncollapse-bouquet}
\end{figure}

\paragraph{Maintaining token count}
In \Cref{sec:empirical_com}, we discussed that the four settings in which metrics stay stable or improve are generally the settings that maintain \textbf{token count for the target languages of interest}.
We elaborate on this here.

With \texttt{fixed-compute}, the token count for each language decreases as we add more languages; thus, mean conditions deteriorate with \texttt{aggregate-all} for both types of sampling.
However, when we have \texttt{fixed-compute} and only look at high-resource languages in \texttt{aggregate-fixed} with \texttt{realistic-sampling}, this effect is somewhat softened since subsequent LRLs consume a smaller part of token budget.
With \texttt{increasing-compute}, token count for target languages remains unchanged with more training languages with \texttt{uniform-sampling}, and metrics stay stable for both types of aggregation. 
With \texttt{increasing-compute} and \texttt{realistic-sampling}, similar to above, HRLs are able to maintain token count, allowing metrics to stay stable for \texttt{aggregate-fixed}, whereas
\texttt{aggregate-all} over \texttt{realistic-sampling} means that that performance means are driven down by low-resource languages.
\comemb trends may potentially also be affected by other things such as model scale, optimization strategies, and data quality, the impact of which we leave to future work to investigate.

\subsection{Additional exploration of language subspace use and dimensionality}
\label[appendix]{app:lang_subspace}

Our construction achieving perfect multilinguality in \Cref{sec:proof_suff} uses compact language offsets that are reused across concepts.
This idea of low-dimensional and reusable language feature representations is of interest to several works \citep{park2023linear,xie-etal-2022-discovering,chang-etal-2022-geometrya}

We want to investigate whether, similarly to our constructive solution, language identity is compactly and reusably represented in our  trained models that maintain the multilinguality conditions, and whether this differs by language scale. 

\paragraph{Defining language subspace efficiency}
We quantify the efficiency of the model regarding language information as the effective dimensionality of the \emph{language subspace}, or the space of the language offsets or language-expressing directions.

\paragraph{Constructing the language subspace}
The language subspace is constructed as the space spanned by the difference vectors between pairs of translation equivalents, i.e. points expressing an identical concept in two different languages
Formally,
$$ U := \mathrm{span}\{ x_{c,l} - x_{c, l'} | c\in \mathcal{C}, l, l' \in \mathcal{L}, l \neq l' \} $$
Here, $x_{c,l}$ is unit-normalized.
We are now interested in the effective dimensionality of $U$. 

\paragraph{Computing effective dimensionality}

We use participation ratio as a measure of effective dimensionality, in line with previous work \citep{barman-etal-2026-geometry}.
Let $M$ stack the language offsets $x_{c,l}-x_{c,l'}$ for all $c$ and unordered pairs $(l,l')$, and let $s_i$ be the singular values of $M$ after row-centering. 
We use
$$d_{\mathrm{lang}}=\frac{\left(\sum_i s_i\right)^2}{\sum_i s_i^2}.$$
This behaves like a soft dimension count.
Intuitively, if all singular values of active directions are equal, it returns the number of active directions. 
Instead, if the spectrum is dominated by a few directions, it returns a smaller value.
We then report $\mathrm{effdim_{lang}} := \frac{d_{\mathrm{lang}}}{D} \in (0,1]$, where $D$ is the ambient dimensionality of the space.

\paragraph{Trends}

\begin{figure}[!ht]
    \begin{minipage}{0.39\linewidth}
    \centering
    \includegraphics[scale=0.9]{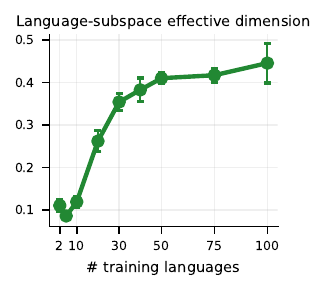}
    \caption{Effective dimensionality of the language subspace of all training languages on \texttt{BOUQuET} for non-cursed training setting. Error bars show standard deviation over five seeds.}
    \label{fig:language-subspace-dimension-bouquet}
    \end{minipage}
    \hfill
    \begin{minipage}{0.56\linewidth}
    \centering
    \includegraphics[width=\linewidth]{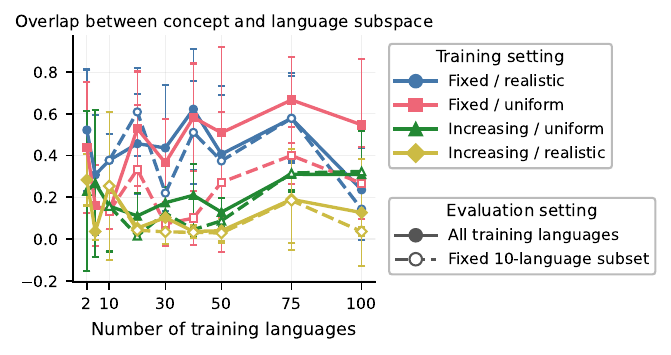}
    \caption{Overlap between concept and language subspace, using \texttt{BOUQuET}. \vspace{3.5em}}
    \label{fig:overlap}
\end{minipage}
\end{figure}

We are interested in $\mathrm{effdim_{lang}}$ for models that maintain the multilinguality conditions.
We construct the language subspace over all training languages, and so we only look at \texttt{aggregate-all} non-cursed configurations.
There is only one such configuration.
See \Cref{fig:language-subspace-dimension-bouquet} for trend of $\mathrm{effdim_{lang}}$ on \texttt{BOUQuET} (similar for other datasets) on this configuration.

We find that the model consumes considerably more effective dimensionality to represent language information with higher language count, going up to 50\% of ambient dimensionality.

\paragraph{Interaction between language and concept subspace}

The above experiment raises a natural next point of comparison with our constructive solution.
Namely: is the language subspace orthogonal to the concept subspace, or are semantic and language directions entangled?
Note that in our constructive solution, the concept and language subspaces are orthogonal by construction.

We construct the language subspace as above, and analogously construct a concept subspace using same-language different-concept pairs.
We then consider the minimal orthonormal bases of both subspaces that capture $k=20\%$ of energy with dimensionality $d_c$ and $d_l$, and quantify the overlap between these directions using principal angle between the subspaces, adjusted for random overlap.
Given $m=\min(d_c,d_l)$, we formulate this as:

$$
O_{\mathrm{CL}}=
\frac{\frac{1}{m}\sum_{i=1}^{m}\sigma_i^2-\frac{\max(d_c,d_l)}{D}}
{1-\frac{\max(d_c,d_l)}{D}}.
$$
where $O_{\mathrm{CL}} \le 1$ and higher values mean more overlap.
Motivating this briefly: it can be shown that if $V_{\mathrm{concept}}\in\mathbb{R}^{D\times d_c}$ and $V_{\mathrm{lang}}\in\mathbb{R}^{D\times d_l}$ are  the orthonormal bases capturing $k$ proportion of energy, then, $\sigma_i=\operatorname{svd}(V_{\mathrm{concept}}^\top V_{\mathrm{lang}})_i=\cos\theta_i$, where $\theta_i$ are principal angles between $V_{\mathrm{concept}}$ and $V_{\mathrm{lang}}$.
It can also be shown that the expected overlap between two subspaces with randomly chosen dimensions given a total dimensionality $D$ is: 
$$\mathbb{E}\!\left[\frac{1}{m}\sum_{i=1}^{m}\cos^2\theta_i\right]=\frac{\max(d_c,d_l)}{D}$$
So our metric $O_\mathrm{CL}$ is a measure of the squared cosines of principal angles between the two subspaces of interest, adjusted for random.

We plot this metric for all settings, shown in \Cref{fig:overlap}. 
We observe that the overlap is generally high across settings and number of training languages, including non-cursed settings.
This is in contrast to the ideal constructive solution, where the concept and language subspace are orthogonal, with zero overlap.

In general, our preliminary experiments indicate that empirical embedding spaces at our small scale do not behave similar to our ideal constructive solution even in non-cursed conditions.
We leave it to future work to explore the underlying mechanisms of the (non-)curse of multilinguality for embedding space structure in conducive settings and at scale, for various architectures and training configurations.

\end{document}